\documentclass{amsart}
\usepackage{url}
\usepackage{array}
\usepackage{subcaption}
\usepackage{mathtools}
\usepackage{multicol}
\usepackage{amsmath}
\usepackage{algorithm}
\usepackage{algorithmic}
\usepackage{amssymb}
\usepackage{xcolor}
\author{Gregory J.~Puleo \and Olgica Milenkovic}
\thanks{This work was supported in part by NSF Grants CIF 1218764, CIF 1117980, 1339388 and STC Class 2010, CCF 0939370, and the Strategic Research Initiative (SRI) Grant conferred by the University of Illinois, Urbana-Champaign. Research of the first author is supported by the IC Postdoctoral Research Fellowship.}
\title[Constrained Cluster Sizes]{Correlation Clustering with Constrained Cluster Sizes and Extended Weights Bounds}
\renewcommand{\subset}{\subseteq}
\DeclareMathOperator{\cost}{cost}

\DeclareMathOperator{\opt}{opt}

\newcommand{\hang}[1]{\makebox[0cm]{#1}}
\newcommand{\hangs}[1]{\hang{\ #1}}

\newcommand{\st}{\colon\,}
\newcommand{\cee}{\mathcal{C}}

\newcommand{\ex}{\mathbb{E}}

\newcommand{\lfrac}[2]{#1/#2}

\newcommand{\caze}[2]{\textbf{Case {#1}:} \textit{#2}}
\newcommand{\sizeof}[1]{\left\lvert{#1}\right\rvert}
\newtheorem{proposition}{Proposition}

\newtheorem{observation}[proposition]{Observation}
\newtheorem{lemma}[proposition]{Lemma}
\newtheorem{theorem}[proposition]{Theorem}

\theoremstyle{definition}
\newtheorem{assumption}{Assumption}

\theoremstyle{remark}

\newcommand{\muu}{\mu^*}
\newcommand{\wpp}{w^+}
\newcommand{\wmm}{w^-}
\newcommand{\pc}{\textsc{pc}}
\newcommand{\ti}{\textsc{ti}}
\newcommand{\maxag}{\textsc{MaxAgree}}
\newcommand{\mindis}{\textsc{MinDisagree}}
\newcommand{\ccpivot}{\textsc{cc-pivot}}
\begin{document}
\maketitle
\begin{abstract}
  We consider the problem of correlation clustering on graphs with
  constraints on both the cluster sizes and the positive and negative
  weights of edges. Our contributions are twofold: First, we
  introduce the problem of correlation clustering with bounded cluster
  sizes. Second, we extend the regime of weight values for which the
  clustering may be performed with constant approximation guarantees
  in polynomial time and apply the results to the bounded cluster size
  problem.
\end{abstract}
\section{Introduction}
The correlation clustering problem was introduced by Bansal, Blum, and
Chawla~\cite{BBC1,BBC2}.  In the most basic form of the clustering
model, one is given a set of objects and, for some pairs of objects,
one is also given an assessment as to whether the objects are
``similar'' or ``dissimilar''. This information is described using a
graph $G$ with labeled edges: each object is represented by a vertex
of the graph, and the assessments are represented by edges labeled
with either a $+$ (for similar objects) or a $-$ (for dissimilar objects).  
The goal is to partition the objects into \emph{clusters}
so that the edges \emph{within} clusters are mostly positive and the
edges \emph{between} clusters are mostly negative. Unlike many other
clustering models, such as $k$-means~\cite{K-means}, the number of
clusters is not fixed ahead of time. Instead, finding the optimal
number of clusters is part of the problem. For this reason,
correlation clustering has been used in machine learning as a model of
``agnostic learning''~\cite{agnostic}.

In a perfect clustering, all the edges within clusters would be
positive and all the edges between clusters would be
negative. However, the similarity assessments need not be mutually consistent: for
example, if $G$ contains a triangle with two positive edges and one
negative edge, then we must either group the endpoints of the negative
edge together (erroneously putting a negative edge inside a cluster, resulting in a
``negative error'') or else we must group them separately (forcing one
of the positive edges to erroneously go between clusters, resulting in a ``positive
error''). When a perfect clustering is not possible, we seek an
\emph{optimal} clustering: one that minimizes the total number of
``errors''. 

Correlation clustering is closely related to a number of other graph
optimization problems, including the \emph{cluster editing} problem,
introduced by Shamir, Sharan, and Tsur~\cite{shamir}; the \emph{graph partitioning} 
problem~\cite{graph-partitioning}; and the \emph{rank aggregation} problem~\cite{rank-aggregation}.

In the cluster editing problem, one is given a graph $H$ and wishes to
determine the smallest number of edge insertions and deletions that
must be carried out in order to transform $H$ into a disjoint union of
cliques. This problem is equivalent to the special case of
correlation clustering where $G$ is a complete graph: labeling the
edges of $H$ with $+$ and the edges not in $H$ with $-$, we see that
the edges added to $H$ correspond to negative errors in a clustering
of $G$, while the edges removed from $H$ correspond to positive
errors. The variant of cluster editing where we are only allowed to
delete edges, rather than adding them, is called the \emph{cluster deletion} 
problem~\cite{shamir}. Cluster editing has a number
of important applications in bioinformatics, including the analysis of
putative gene co-regulation~\cite{cluster-editing-application}.

In the graph partitioning problem, one is asked to partition the
vertices of a graph $H$ into a fixed number of parts, minimizing the
number of edges between different parts~\cite{graph-partitioning}
subject to an upper bound on the size of the parts. If we again
consider a complete graph $G$ whose positive edges are the
edges of $H$, we can view graph partitioning as a variant of
correlation clustering where the number of clusters is fixed, only
``positive errors'' are penalized, and there is an \emph{upper bound} on the
size of clusters allowed in a solution.

In the rank aggregation problem, and in particular, the \emph{feedback
  arc set} problem on tournaments, one seeks a globally consistent ranking (i.e.,
permutation) based on possibly inconsistent pairwise
comparisons~\cite{feedback-arc-set}. Equivalently, the problem is to
compute the centroid of a set of permutations under a suitable ordinal
distance measure, usually taken to be the Kendall or weighted Kendall
distance~\cite{weighted-kendall}. Ailon, Charikar, and
Newman~\cite{ACN1,ACN2} showed that essentially the same algorithm may
be used to obtain good constant-factor approximations for both the
correlation-clustering problem and the rank-aggregation
problem.

Finding optimal solutions to all the aforementioned problems is
computationally hard.  In particular, optimal correlation clustering
is computationally hard even when $G$ is a complete graph: Bansal,
Blum, and Chawla~\cite{BBC1,BBC2} proved that it is NP-hard to
determine whether a labeled complete graph $G$ has a clustering with
at most $k$ errors, where $k$ is an input parameter. This was also
proved independently by Shamir, Sharan, and Tsur~\cite{shamir} in the
context of cluster editing.  Shamir, Sharan, and Tsur also proved that
for some $\epsilon > 0$, the cluster deletion problem is NP-hard to
approximate within a factor of $(1+\epsilon)$. Rank aggregation under
the Kendall $\tau$-distance was proved to be NP-hard by
Bartholdi~\cite{hard-voting}.  Research has therefore focused on
finding efficient approximation algorithms for this family of
problems.

The objective function for correlation clustering can be viewed in two
different ways, leading to two different formulations of the problem
with different approximability properties. In the \maxag{}
formulation, we seek a clustering that \emph{maximizes} the number of
edges which agree with the clustering. Bansal, Blum, and
Chawla~\cite{BBC1, BBC2} gave a PTAS for \maxag{} in the case where
$G$ is a complete graph. On general graphs, \maxag{} is
APX-hard~\cite{CGW1,CGW2}, although constant-factor approximation
algorithms are known~\cite{CGW1,CGW2,agnostic}.

In our work, we focus on the \mindis{} formulation. Here, we seek a
clustering that \emph{minimizes} the number of edges which disagree
with the clustering. \mindis{} on general graphs is APX-hard, and the
best known approximation ratio is $O(\log n)$, given
by~\cite{CGW1,CGW2,DEFI1}. Any improvement on this approximation ratio
would also improve the approximation ratio for the minimum multicut
problem~\cite{min-multicat}, as shown in
\cite{CGW1,CGW2,DEFI1}. Assuming the Unique Games Conjecture of
Khot~\cite{khotUGC}, there is no constant-factor approximation
algorithm for the minimum multicut problem~\cite{UGC1,UGC2}, so the
same conditional hardness result holds for correlation clustering on
general graphs.

However, constant-factor approximation algorithms for \mindis{} have
been found for the case where $G$ is a complete graph. In their
original paper, Bansal, Blum, and Chawla~\cite{BBC1,BBC2} gave a
$17429$-approximation algorithm.  Charikar, Guruswami, and
Wirth~\cite{CGW1,CGW2} gave a $4$-approximation algorithm for the
problem, based on a two-stage procedure that rounds the solution of a
linear program. Applying a similar algorithm to general graphs yields
a $O(\log n)$-approximation~\cite{CGW1,CGW2}; this was also
independently proved by Demaine, Emanuel, Fiat, and
Immorlica~\cite{DEFI1,DEFI2,DEFI3}. In some sense, this $O(\log n)$ factor gives 
the best possible approximation for techniques based on linear programming, since
the natural linear programming formulation has integrality gap
$\Omega(\log n)$ on general graphs~\cite{DEFI1,DEFI2,DEFI3}.

Later on, Ailon, Charikar, and Newman~\cite{ACN1,ACN2} introduced a
simple randomized $3$-approximation algorithm for \mindis{} on
complete graphs, i.e., an algorithm where the \emph{expected} cost of
the output is at most $3$ times the optimal cost. They also gave a
randomized LP-rounding algorithm achieving an expected approximation
ratio of $2.5$. Van Zuylen and Williamson~\cite{VZW} gave derandomized
versions of these algorithms, achieving the same approximation
ratios. More recently, Chawla, Makarychev, Schramm and
Yaroslavtev~\cite{near-optimal} announced a refinement of the
randomized LP-rounding algorithm as well as a derandomization scheme,
achieving a deterministic approximation ratio of $2.06$ on complete
graphs.  Since the integrality gap of the LP in question is known to
be at least $2$, this approximation ratio is very close to the best
ratio possible using techniques based on LP rounding.

Bansal, Blum, and Chawla~\cite{BBC1,BBC2} also proposed a weighted version
of the correlation-clustering problem, where the edges of the graph
$G$ receive \emph{weights} between $-1$ and $1$ rather than simply
receiving $+$ or $-$ labels: an edge with weight $x$ incurs cost
$\frac{1+x}{2}$ if it is placed between clusters and cost
$\frac{1-x}{2}$ if it is placed within a cluster. Another weighted
formulation was considered by Demaine, Emanuel, Fiat, and
Immorlica~\cite{DEFI1,DEFI2,DEFI3}: here, the edges are still labeled $+$ or $-$,
but each edge also receives a nonnegative weight $w$, with positive
edges costing $w$ when placed between clusters and negative edges
costing $w$ when placed within clusters.

A more general weighted formulation was introduced by Ailon, Charikar,
and Newman, and this is the formulation we subsequently consider.  In
the Ailon--Charikar--Newman model, each edge $e$ is assigned two
nonnegative weights, $\wpp_e$ and $\wmm_e$. A clustering incurs cost
$\wpp_e$ if $e$ is placed between clusters, and incurs cost $\wmm_e$
if $e$ is placed within a cluster. This model includes both the
Bansal--Blum--Chawla formulation and the
Demaine--Emanuel--Fiat--Immorlica formulation as special cases. The
cluster deletion problem can also be represented in this framework, by
giving negative edges a prohibitively high value of $\wmm_e$.

If no restrictions are placed on the weights $\wpp_e$ and $\wmm_e$,
then it is possible to have edges with $\wpp_e = \wmm_e = 0$; these
edges are effectively absent from the graph, so there is no loss of
generality in assuming that $G$ is a complete graph.  Since
correlation clustering is hard to approximate on general graphs, this
suggests that we must place restrictions on $\wpp_e$ and $\wmm_e$ in
order to obtain a class of problems for which a constant-factor
approximation is possible.

The \emph{probability constraints}, which we abbreviate to \pc, give a
natural restriction on the edge weights: in this case, we require
$\wpp_e + \wmm_e = 1$ for every edge $e$. The weight functions
satisfying \pc{} are exactly the weight functions that can be
represented in the Bansal--Blum--Chawla weighting model.  Bansal,
Blum, and Chawla~\cite{BBC1,BBC2} showed that any $k$-approximation
algorithm for the unweighted correlation-clustering problem yields a
$(2k+1)$-approximation algorithm to the weighted problem under \pc.
In conjunction with the $4$-approximation algorithm of Charikar,
Guruswami, and Wirth for the unweighted problem, this yields a
$9$-approximation algorithm for the weighted problem under \pc.
Ailon, Charikar, and Newman obtained a randomized expected
$5$-approximation algorithm for the weighted problem under \pc.

Another widely studied restriction is the \emph{triangle inequality}
restriction (abbreviated \ti), where one requires $\wmm_{uw} \leq
\wmm_{uv} + \wmm_{vw}$ for all $u,v,w \in V(G)$.  Such constraints
arise naturally when correlation clustering is used as a model for
clustering aggregation. Gionis, Mannila, and Tsaparas~\cite{GMT}
obtained a $3$-approximation algorithm for correlation clustering
under \pc{} and \ti, while Ailon, Charikar, and Newman obtained a
randomized expected $2$-approximation algorithm under the same
constraints. Little appears to be known about the approximability of
correlation clustering with \emph{only} \ti, without \pc{} or other
constraints involved.  Indeed, Ailon, Charikar, and Newman leave this
approximability question as an open problem. This question was
partially addressed by van Zuylen and Williamson~\cite{VZW}, who
considered weighted correlation clustering under \ti{}-like hypotheses
stronger than the version of \ti{} stated above.

Previous treatments of the correlation-clustering problem have allowed
\emph{any} clustering to be considered as a feasible
solution. However, in some applications there may be restrictions on
which clusterings can actually be meaningful as solutions. In
particular, the application may demand that no cluster should be very
large: say, each cluster has size at most $K+1$, for some fixed
integer $K$. Such constraints frequently arise in community detection
problems, as described
in~\cite{complexity,bmc,milenkovic2010introduction,communities,yarkony2012fast,zhang2014cell},
where one may have a priori information about the largest possible
size of a subcommunity.  Similar constraints were also considered for
clustering models unrelated to correlation clustering by Khandekar,
Hildrum, Parekh, Rajan, Sethuraman, and Wolf~\cite{boundedstream}. As
discussed earlier, an upper bound on cluster sizes also appears in the
graph partitioning problem~\cite{graph-partitioning}. The
approximation algorithms we present in the next sections are able to
accommodate such constraints in the correlation-clustering case, as
well as accommodate ``soft constraints'' that merely impose a
penalty on the objective function for oversized clusters, without
outright forbidding them.

\section{Our Contributions}
The contributions of our work are two-fold. 
First, we generalize the $4$-approximation algorithm of
Charikar, Guruswami, and Wirth~\cite{CGW1,CGW2} in order to handle
more general weighted graphs than previously considered. When $G$ is a
complete graph, and when the weight functions satisfy the following
constraints:
\begin{itemize}
\item $\wpp_e \leq 1$ for every edge $e$, and
\item $\wmm_e \leq \tau$ for every edge $e$, for some $\tau \in [1, \infty)$, and
\item $\wpp_e + \wmm_e \geq 1$ for every edge $e$,
\end{itemize}
our algorithm achieves an approximation ratio of $5 - 1/\tau$.  Thus,
when the weights satisfy the probability constraints $\wpp_e + \wmm_e
= 1$ for all $e$, the algorithm guarantees a $4$-approximation ratio
and reduces to the Charikar--Guruswami--Wirth approach. (It does not
appear to have been previously recognized that the
Charikar--Guruswami--Wirth algorithm, originally stated for unweighted
instances, works without modification for weights obeying the
probability constraints and gives the same approximation ratio;
cf. Table~1 of \cite{ACN2}.) In the limit $\tau \to \infty$ where
$\wmm_e$ is allowed to be arbitrarily large, the algorithm still
guarantees a constant approximation ratio of $5$.

Note, in particular, that the cluster-deletion problem for a graph $G$
is equivalent to the correlation-clustering problem for the graph $G'$
where all edges $e \in E(G)$ satisfy $\wpp_e =1, \wmm_e = 0$ and all
edges $e \in E(\overline{G})$ satisfy $\wpp_e = 0, \wmm_e = \infty$.
Charikar, Guruswami, and Wirth showed that a slight modification of
their algorithm gives a $4$-approximation for cluster
deletion~\cite{CGW1,CGW2}; in contrast, our algorithm gives a
$5$-approximation but also applies to problems that are intermediate
between correlation-clustering and cluster-deletion, where negative
errors are permissible but may be costly.

Our second contribution is to extend the correlation-clustering model
to allow for bounds on the sizes of the clusters. We assume that there
is a ``soft constraint'' of at most $K+1$ vertices per cluster. To
model this constraint, we give each vertex $v$ a ``penalty'' parameter
$\mu_v$.  If $v$ is placed in a cluster $C$ with more than $K+1$
vertices, then the cost function is charged an additional penalty of
$\mu_v(\sizeof{C} - (K+1))$.  Since this penalty is assessed
separately for each vertex in an oversized cluster, the total penalty
cost for a cluster $C$ with $\sizeof{C} > K+1$ is given by
$(\sizeof{C} - (K+1))\sum_{v \in C}\mu_v$.

If instead a ``hard constraint'' is desired, it suffices to set all 
the penalties to $\mu_v=1$: any clusters in the resulting solution which
are too large can then be split arbitrarily into clusters of size $K+1$ and a ``remainder cluster'', 
yielding no net increase in cost since we assume that
$\wpp_e \leq 1$ for all $e$. Similarly, if no size constraint at all
is desired, we can take all $\mu_v=0$.

In addition to the changes mentioned previously, we show that the
Charikar--Guruswami--Wirth algorithm can be further modified in order
to deal with both the weight and cluster size constraints. Taking $\muu = \max_{u \neq
  v}(\mu_u + \mu_v)$, we obtain an
approximation ratio of $\max\{\muu, 2/\alpha\}$ for the aforementioned clustering scenario, where $\alpha$ is the unique
positive solution to the equation
\[ \frac{2\alpha\muu}{1-2\alpha} + \frac{1}{1-2\alpha+\frac{\alpha}{2\tau}} = \frac{2}{\alpha\hangs.}\]
This rearranges to a cubic equation in $\alpha$ which has an unwieldy analytical solution. However,
for some values of $\tau$ and $\muu$, there is a simple formula for the approximation ratio;
these values are summarized in Table~\ref{tab:appxratio}. In particular, since we only consider
$\alpha \leq 1/2$, whenever $\muu \leq 4$ we have $\max\{\muu, 2/\alpha\} = 2/\alpha$.
\begin{table}
  \centering
  \begin{tabular}{c|c|c|c}
    & $\tau=1$ & $\tau \in [1,\infty)$ & $\tau \to \infty$  \\
    & (generalizes && (arbitrarily large $\wmm_e$)\\
    & prob.~constraints) & & \\\hline
    $\muu=0$ & 4 & $5 - (1/\tau)$ & 5 \\
    (no size bound) &  & & \\\hline
    $\muu \in (0,2)$ & $\star$ & $\star$ & $\frac{8\muu}{-5 + \sqrt{25+16\muu}}$ \\
    (soft size bound) &  &  & \\\hline
    $\muu=2$ & 6 & $\star$ & $\approx 6.275$ \\
    (hard size bound) &&&
  \end{tabular}
  \vspace{0.1in}
  \caption{Approximation ratios for special parameter values. Starred entries are solutions to a cubic equation.}
  \label{tab:appxratio}
\end{table}

We briefly note a curious asymmetry in our result. Our algorithm is
capable of handling size-bounded clustering instances with arbitrarily
large negative weights, but cannot handle instances with very large
positive weights. Since clustering is hard when the weights are
allowed to be arbitrary, one cannot reasonably hope to handle
instances with \emph{both} large negative weights and large positive
weights, but it is reasonable to ask whether there is a
constant-factor approximation algorithm for instances with arbitrarily
large positive weights and bounded negative weights -- the mirror
image, in some sense, of the instances handled by our algorithm. The
hardness of the balanced graph partitioning problem suggests that no
such algorithm is possible: since only positive errors matter in the
graph partitioning problem, we can roughly model the graph
partitioning problem as a weighted correlation-clustering problem by
giving each positive edge a weight of (say) $1000\sizeof{E(G)}$ and
each negative edge weight $1$, so that a single positive error costs
much more than all possible negative errors combined. Andreev and
Racke \cite{graph-partitioning} proved that it is NP-hard to achieve
any finite approximation ratio for the graph partitioning problem when
we require a partition into $k$ clusters, each of size exactly
$n/k$. This \emph{suggests} that it should also be NP-hard to solve
the corresponding correlation-clustering instances with an upper bound
of size $n/k$ for clusters. However, the analogy is not perfect, since
solutions to the correlation-clustering instance may use more than $k$
clusters, and we do not pursue the question any further.
\section{A Region-Growing Algorithm} \label{sec:two-step} Throughout
this and subsequent sections, we denote vertices of the complete graph $G$ 
by lowercase letters. We use the symbol $e$ to denote a
generic edge in $E(G)$, and when needed, specify the endpoint
vertices of the edge as $e=uv$, $u,v \in V(G)$. 
As before, $\wpp_e$
and $\wmm_e$ denote the positive and negative weight of the edge $e$,
while $\tau$ represents an upper bound on the weight of negatively
labeled edges as specified in the introduction. 

Our algorithm builds upon the Charikar--Guruswami--Wirth
algorithm~\cite{CGW1,CGW2} but introduces a number of changes
both in the LP formulation and region-growing procedure.  We require
the following assumptions on the weights and the graph $G$:
\begin{assumption}\label{assumption}  
  We assume that $uv \in E(G)$ for every pair of distinct vertices $u,v \in V(G)$.
  For every edge $e$ we assume that $\wpp_e \leq 1$, that
   $\wmm_e \leq \tau$, where $\tau \in [1, \infty)$, and that $\wpp_e + \wmm_e \geq 1$.
\end{assumption}
\begin{figure}
  \[\begin{aligned}
    & \underset{x,y}{\text{minimize}}
    & & \mathrlap{\left[\sum_{e \in E(G)}(\wpp_e x_e + \wmm_e (1-x_e))\right] + \sum_{v \in V(G)}\mu_vy_v} \\
    & \text{subject to}
    & & x_{uv} \leq x_{uz} + x_{zv} &\quad{\text{(for all distinct $u,v,z \in V(G)$)}} \\
    &&&  \sum_{v \neq u}(1 - x_{uv}) \leq K + y_u &\quad{\text{(for all $u \in V(G)$)}}\\
    &&&  0 \leq x_{e} \leq 1 &\quad{\text{(for all $e \in E(G)$)}}\\
    &&&  0 \leq y_{v} &\quad{\text{(for all $v \in V(G)$)}}
  \end{aligned}\]
  \caption{The linear program $P$.}
  \label{fig:LP}
\end{figure}

Let $P$ denote the linear program shown in Figure~\ref{fig:LP}. The
integer restriction of $P$, where $x_{e} \in \{0,1\}$, represents the
exact model for our \emph{weighted and cluster-size-bounded}
correlation-clustering problem. We interpret $x_e = 1$ as meaning
``the endpoints of $e$ lie in different clusters'' and we interpret
$x_e = 0$ as meaning ``the endpoints of $e$ lie in the same cluster''.
The triangle inequality $x_{uv} \leq x_{uz} + x_{zv}$ models the fact
that if two edges $uz$ and $zv$ are in the same cluster, then the edge
$uv$ should also belong to the same cluster. The new restriction on
cluster sizes is represented by the constraints $\sum_{v \neq u}(1 -
x_{uv}) \leq K + y_v$ together with the penalty term $\sum_{v \in
  V(G)}\mu_vy_v$ in the objective function; here, $y_v$ represents the
amount by which the cluster containing $v$ exceeds the size bound. To
simplify notation, we adopt the convention that $x_{uu}=0$ for all
$u$.

Since any actual clustering yields a feasible integer solution to $P$,
the optimal value of $P$ is a lower bound for the optimal cost of the
correlation-clustering problem. The idea behind the algorithm is to start
with an optimal solution to $P$ and ``round'' the solution to produce
a clustering. We prove that this rounding process only increases the cost
of the solution by a constant multiplicative factor.

Given any feasible solution $(x,y)$ to the linear program, where $x$
denotes the vector of all edge costs $x_e$, while $y$ denotes the
vector of all vertex penalties $y_v$, we produce a clustering $\cee$
via the following algorithm, where $\alpha \in (0, 1/2]$ is a parameter to be
determined later.
\begin{algorithm}
\caption{Round LP solution to obtain clustering, using a threshold parameter $\alpha \in (0,1/2]$.}
\label{alg:round}  
\begin{algorithmic}
\STATE{Let $S = V(G)$.}
\WHILE{$S \neq \emptyset$}
\STATE{Let the ``pivot vertex'' $u$ be an arbitrary element of $S$.}
\STATE{Let $T = \{w \in S-\{u\} \st x_{uw} \leq \alpha\}$.}
\IF{$\sum_{w \in T}x_{uw} \geq \alpha\sizeof{T}/2$}
\STATE{Output the singleton cluster $\{u\}$.}
\STATE{Let $S = S-\{u\}$.}
\ELSE
\STATE{Output the cluster $\{u\} \cup T$.}
\STATE{Let $S = S - (\{u\} \cup T)$.}
\ENDIF
\ENDWHILE  
\end{algorithmic}
\end{algorithm}
Note that the only difference between Algorithm~\ref{alg:round} and
the second stage of the Charikar--Guruswami--Wirth algorithm is the use of the parameter
$\alpha$ rather than the fixed value of $1/2$. As we subsequently prove, the choice of this
threshold allows for extensions in the range of values for the weights of edges that can
be accommodated by our algorithm.

We show next that the resulting clustering $\cee$ has cost at most
$c_{\alpha}\cost(x,y)$, where $c_{\alpha}$ is a constant depending
only on $\alpha$ and $\cost(x,y)$ denotes the cost of the original
solution of the linear program.  We follow the outline of the analysis
performed by Charikar--Guruswami--Wirth~\cite{CGW1, CGW2}, while
incorporating the constraints on weights and cluster sizes at hand.

In our derivations, we make frequent use of the \emph{LP-cost} and
\emph{cluster-cost} of an edge $e$: the LP-cost refers to the cost of that
edge using the LP solution $(x,y)$, i.e., the quantity $\wpp_ex_e +
\wmm_e(1-x_e)$, while the cluster-cost refers to the actual cost
incurred by that edge in the clustering produced by
Algorithm~\ref{alg:round}.
\begin{observation}\label{obs}
  Let $x$ be a feasible solution to $P$, and let $wz$ be an edge.
  For any vertex $u$, we have $x_{wz} \geq x_{uz} - x_{uw}$ and
  $1 - x_{wz} \geq 1 - x_{uz} - x_{uw}$.
\end{observation}
The next lemma bounds the LP-cost of a subset of edges used in the region-growing Algorithm~\ref{alg:round}.
All notation is per the definition of the algorithm.
\begin{lemma}\label{lem:Rcost}
  Let $(x,y)$ be a feasible solution to $P$.  Suppose that $z \in S$ at
  the beginning of the iteration, let $u$ be the pivot vertex, and let
  $R \subset \{u\} \cup T$. For any $\zeta \in [0,1]$, if $\sum_{v \in R}x_{uv} \leq \lfrac{\alpha\sizeof{R}}{2}$ and $x_{uv} \leq \zeta$ for all $v \in R$, then the total LP-cost of the edges
  joining $z$ and $R$ is at least
  \[ \sum_{v \in R}\left[\wpp_{vz}x_{uz} + \wmm_{vz}(1 - x_{uz}) - \zeta(\wpp_{vz} + \wmm_{vz}) + (\zeta-\frac{\alpha}{2})\right]. \]
\end{lemma}
\begin{proof}
  By Observation~\ref{obs}, we have $x_{vz} \geq x_{uz}-x_{uv}$ and $1 - x_{vz} \geq 1 - x_{uz} - x_{uv}$
  for each edge $vz$. Thus, we have the following lower bound on the total LP-weight of the edges joining $z$ and $R$:
  \begin{align*}
    \sum_{v \in R}\left[\wpp_{vz}x_{vz} + \wmm_{vz}(1-x_{vz})\right] &\geq \sum_{v \in R}\left[\wpp_{vz}(x_{uz} - x_{uv}) + \wmm_{vz}(1 - x_{uz} - x_{uv})\right] \\
    &= \sum_{v \in R}\left[\wpp_{vz}x_{uz} + \wmm_{vz}(1 - x_{uz})\right] - \sum_{v \in R}(\wpp_{vz} + \wmm_{vz})x_{uv}.
  \end{align*}
  Using $\sum_{v \in R}x_{uv} \leq \alpha\sizeof{R}/2$ and $x_{uv} \leq \zeta$ for $v \in R$, we bound $\sum_{v \in R}(\wpp_{vz} + \wmm_{vz})x_{uv}$ as follows:
  \begin{align*}
    \sum_{v \in R} (\wpp_{vz} + \wmm_{vz})x_{uv} &= \sum_{v \in R}x_{uv} + \sum_{v \in R} (\wpp_{vz} + \wmm_{vz} - 1)x_{uv} \\
    &\leq \frac{\alpha\sizeof{R}}{2} + \sum_{v \in R} (\wpp_{vz} + \wmm_{vz} - 1)x_{uv} \\
    &= \sum_{v \in R} \left[(\wpp_{vz} + \wmm_{vz} - 1)x_{uv} + \frac{\alpha}{2}\right] \\
    &\leq \sum_{v \in R} \left[\zeta(\wpp_{vz} + \wmm_{vz}) + (\frac{\alpha}{2}-\zeta)\right].
  \end{align*}
  Note that in the last inequality we used the fact that $\wpp_{vz} + \wmm_{vz} \geq 1$.
\end{proof}
\begin{theorem}\label{thm:appx}
  Let $G$ be a correlation-clustering instance satisfying
  Assumption~\ref{assumption}, and let $\alpha \in (0, 1/2)$.  If $(x,y)$ is any
  solution to the linear program $P$, then the clustering produced by
  Algorithm~\ref{alg:round} using input $(x,y)$ and parameter $\alpha$ has cost at most
  $c_\alpha\cost(x,y)$, where
  \[ c_\alpha = \max\left\{\muu,\ \frac{2\alpha\muu}{1-2\alpha} + \frac{1}{1 - 2\alpha + \frac{\alpha}{2\tau}},\ \frac{2}{\alpha}\right\} \]
  and $\muu = \max_{u \neq v}(\mu_u + \mu_v)$. Also, if $\muu = 0$ and $\alpha = 1/2$, then
  the same conclusion holds with $c_\alpha = \max\left\{\frac{1}{1-2\alpha + \frac{\alpha}{2\tau}},\ \frac{2}{\alpha}\right\}$.
\end{theorem}
\begin{proof}
  The idea behind the proof is as follows. After each cluster output $\{u\} \cup T$ is produced, one is asked
  to ``pay
  for'' the cluster-cost of the edges which are both incident to
  $\{u\} \cup T$ and were contained within the pool of vertices $S$ available at the beginning of
  the given iteration of the algorithm. We pay for the cluster-cost by
  ``charging'' some quantity to each of these edges. Subsequently, we show that
  the total amount charged to each edge is at most $c_{\alpha}$ times
  its LP-cost.  While some edges may be charged multiple times, one can 
  appropriately bound the \emph{total} amount of charge accrued
  by each edge.

  Suppose that some cluster $\{u\} \cup T$ has just been selected as an output of Algorithm 1.
  We split the analysis of the claimed result into two cases, according to whether or not $\{u\} \cup T$ is a
  singleton set.

  \caze{1}{The output is a singleton cluster $\{u\}$.} The total cluster
  cost when outputting a singleton cluster $\{u\}$ is $\sum_{v \in S-u}
  \wpp_{uv}$, while the total LP-cost accrued by edges incident to $u$
  is $\sum_{v \in S-u} \wpp_{uv}x_{uv}$.

  If the singleton $\{u\}$ is output, then we have
  \[ \sum_{v \in T}x_{uv} \geq \frac{\alpha\sizeof{T}}{2}. \] For each $v \in T$, we have $x_{uv} \leq \alpha$. For such
  $x_{uv}$, we also have $1-x_{uv} \geq x_{uv}$, since $\alpha < 1/2$. This
  yields the following lower bound on the LP-cost of $uv$:
  \[ \wpp_{uv} x_{uv} + \wmm_{uv} (1-x_{uv}) \geq \wpp_{uv} x_{uv} +
  \wmm_{uv} x_{uv} \geq x_{uv}, \] where the last inequality follows from the 
  bound $\wpp_{uv} + \wmm_{uv} \geq 1$. Thus, each edge $uv$ has
  LP-cost at least $x_{uv}$, and so the edges joining $u$ and $T$ have
  total LP-cost at least $\alpha\sizeof{T}/2$. Each such edge $uv$
  incurs cluster-cost $\wpp_{uv}$, which is at most equal to $1$.  Thus,
  charging $(2/\alpha)x_{uv}$ to each edge $uv$ for $v \in T$ produces enough ``charge'' 
  to pay for the cluster-cost of edges $v \in T$, and at the same time, each edge is charged at most $2/\alpha$ times its LP-cost.

  For $v \in S-(T \cup \{u\})$, we have $x_{uv} > \alpha$, so each edge
  $uv$ incurs an LP-cost at least $\alpha \wpp_{uv}$ and a cluster-cost at
  most $\wpp_{uv}$. Thus, charging such edge $(1/\alpha)\wpp_{uv}x_{uv}$ compensates for the cluster-cost of the edges under consideration.

  \smallskip
  \caze{2}{The output is a nonsingleton cluster $\{u\} \cup T$.} We
  first consider edges in $\{u\} \cup T$, and then consider edges
  joining $\{u\} \cup T$ with $S - (\{u\} \cup T)$. We conclude by
  describing how to pay the possible penalties for allowing the set
  $T$ to be too large. Let $\gamma = \alpha - \lfrac{\alpha}{4\tau}$,
  and note that $\lfrac{3\alpha}{4} \leq \gamma \leq \alpha$.

  \textbf{Edges within $\{u\} \cup T$.}  Assume that the vertices of
  $\{u\} \cup T$ are linearly ordered so that $v \leq z$ implies
  $x_{uv} \leq x_{uz}$. For each $z \in T$, let $R_z = \{v \in \{u\}
  \cup T \st v < z\}$, and let $E_z = \{vz \st v \in R_z\}$. Each edge
  contained in $\{u\} \cup T$ appears in exactly one set $E_z$, so for
  each $z \in T$, we show how to pay for the edges in $E_z$.

  If $x_{uz} \leq \gamma$, then $x_{uv} \leq \gamma$ for all $v \in R_z$. 
  Each edge $vz \in E_z$ incurs cluster-cost $\wmm_{vz}$ and LP-cost at least $\wmm_{vz}(1-x_{vz})$,
  so for each $vz \in E_z$, we have
  \[ 1 - x_{vz} \geq 1 - x_{uv} - x_{vz} \geq 1 - 2\gamma, \]
  so that charging $\frac{1}{1 - 2\gamma}\wmm_{vz}(1-x_{vz})$ to the edge $vz$ pays
  for its cluster-cost.

  Otherwise, $x_{uz} > \gamma$. The ordering on $\{u\} \cup T$ implies
  that $\sum_{v \in R_z}x_{uv} \leq \lfrac{\alpha\sizeof{R_z}}{2}$ and
  that $x_{uv} \leq x_{uz}$ for all $v \in R_z$. Applying
  Lemma~\ref{lem:Rcost} with $R = R_z$ and $\zeta = x_{uz}$ and rearranging the
  terms in the sum gives the following lower bound on total LP-cost of the edges in $E_z$:
  \begin{align*}\sum_{v \in R_z}\left[\wpp_{vz}x_{vz} + \wmm_{vz}(1-x_{vz})\right] &\geq 
    \sum_{v \in R_z}\Big[\wpp_{vz}x_{uz} + \wmm_{vz}(1-x_{uz}) - \\
    &\qquad\qquad\quad x_{uz}(\wpp_{vz} + \wmm_{vz}) + (x_{uz} - \frac{\alpha}{2})\Big] \\
      &=\sum_{v \in R_z}\left[\wmm_{vz}(1-2x_{uz}) + (x_{uz} - \frac{\alpha}{2})\right].
\end{align*}
This lower bound is linear in $x_{uz}$. When $x_{uz} = \gamma$, we see that the
  total LP-cost of $E_z$ is at least $(1-2\gamma)\sum_{v \in R_z}\wmm_{vz}$. When $x_{uz} = \alpha$,
  the total LP-cost of $E_z$ is at least
  \[ \sum_{v \in R}\wmm_{vz}(1 - 2\alpha + \frac{\alpha}{2\tau}), \]
  which is equal to $(1-2\gamma)\sum_{v \in R}\wmm_{vz}$ by the choice of $\gamma$.
  Thus, charging $\frac{1}{1-2\gamma}\wmm_{vz}(1-x_{vz})$ to each edge $vz \in E_z$ pays for the
  total cluster-cost of the edges in $E_z$.

  \textbf{Edges joining $\{u\} \cup T$ with $S - (\{u\} \cup T)$.} Let
  $z$ be a vertex such that $z \notin \{u\} \cup T$. A \emph{cross-edge} for $z$
  is an edge from $z$ to $\{u\} \cup T$. We show that the cross-edges
  for $z$ have total cluster-cost that is at most $\max\{\frac{1}{1-2\alpha}, \frac{2}{\alpha}\}$
  times their total LP-cost.  Note that whenever $vz$ is a cross-edge,
  we have $x_{uv} \leq \alpha$, by the definition of $T$. Each
  cross-edge $vz$ incurs cluster-cost $\wpp_{vz}$ and LP-cost
  $\wpp_{vz}x_{vz} + \wmm_{vz}(1-x_{vz})$.

  Let $\eta = \alpha - \frac{\alpha}{2\tau}$. If $x_{uz} \geq 1-\eta$,
  then based on Observation~\ref{obs}, we have
  \[x_{vz} \geq x_{uz} - x_{uv} \geq 1 - \eta - \alpha = 1-2\gamma \]
  for every cross-edge $vz$.  These edges have LP-cost at least $(1 -
  \eta - \alpha)\wpp_{vz}$, and therefore have cluster-cost at most
  $\frac{1}{1 - 2\gamma}$ times their LP-cost.

  It remains to handle the case $x_{uz} \in (\alpha, 1-\eta)$. Here,
  we seek a lower bound on the total LP-cost of the cross-edges for
  $z$. Note that the total cluster-cost of these edges is $\sum_{v \in
    \{u\} \cup T}\wpp_{vz}$, which is at most $\sizeof{T} + 1$ since each
  $\wpp_{vz} \leq 1$.

  Using Lemma~\ref{lem:Rcost} with $R = \{u\} \cup T$ and $\zeta=\alpha$,
  we see that the total LP-cost of the cross-edges for $z$ is at least
  \[ \sum_{v \in \{u\} \cup T}\left[\wpp_{vz}x_{uz} + \wmm_{vz}(1 - x_{uz}) - \alpha(\wpp_{vz} + \wmm_{vz}) + \frac{\alpha}{2}\right].\]
  This lower bound is a linear function in $x_{uz}$. We consider next the behavior of this function on the interval $(\alpha, 1-\eta)$.

  When $x_{uz} = \alpha$, the lower bound simplifies to
  \[ \sum_{v \in \{u\} \cup T}\left[\alpha \wpp_{vz} + (1-\alpha)\wmm_{vz} - \alpha(\wpp_{vz} + \wmm_{vz}) + \frac{\alpha}{2}\right], \]
  which is at least $\alpha(\sizeof{T}+1)/2$ since $\alpha < 1/2$ implies $(1-\alpha)\wmm_{vz} \geq \alpha\wmm_{vz}$.

  When $x_{uz} = 1 - \eta$, the lower bound simplifies to
  \begin{align*}
    \sum_{v \in \{u\} \cup T}\left[ (1-\eta-\alpha)\wpp_{vz} - \frac{\alpha}{2\tau}\wmm_{vz} + \frac{\alpha}{2}\right] 
    &\geq \sum_{v \in \{u\} \cup T}\left[ (1-\eta-\alpha)\wpp_{vz} \right] \\
    &= \sum_{v \in \{u\} \cup T}\left[ (1-2\gamma)\wpp_{vz} \right],
  \end{align*}
  where we used $\wmm_{vz} \leq \tau$. In all cases, charging
  $\max\{\frac{1}{1-2\gamma}, \frac{2}{\alpha}\}$ times the LP-cost of
  the edges pays for their cluster-cost.

  \textbf{Paying for vertex penalties.} Finally, if $\sizeof{\{u\}
    \cup T} > K+1$, then we must pay for the ``penalty cost'' incurred
  by the cluster $\{u\} \cup T$. If $\muu = 0$, then there are no penalties
  to pay, so we may assume that $\muu > 0$ and $\alpha < 1/2$.

  Fix some vertex $z \in \{u\} \cup T$, and let $T_z = (\{u\} \cup T)
  - \{z\}$; note that $\sizeof{T_z} = \sizeof{T}$. The LP-penalty paid
  by $z$ is $\mu_z y_z$, while the cluster-penalty paid by $z$ is
  $\mu_z(\sizeof{T} - K)$. The idea is to charge the edges incident to
  $z$, as well as the vertex $z$ itself, in order to pay the
  cluster-penalty.

  Let $\beta = \frac{1-2\alpha}{2\alpha}$, and let $E_z$ be the edge set defined by
  \[ E_z = \{vz \st v \in T_z\}.\]
  We charge $\mu_z(\wpp_{vz}x_{vz} + \frac{1}{\beta}\wmm_{vz}(1-x_{vz}))$ to each edge $vz \in E_z$.
  Observe that each edge contained in $\{u\} \cup T$ is only charged this way at its endpoints.
  The total charge from the edges in $E_z$ is
  \[ \mu_z\left[\sum_{v \in T_z}\wpp_{vz}x_{vz} + \frac{1}{\beta}\sum_{v \in T_z}\wmm_{vz}(1-x_{vz})\right]. \]
  We also charge $\mu_zy_z$ to the vertex $z$ itself.
  We wish to show that the total charge from the edges in $E_z$ is at least $\mu_z(\sizeof{T}-K-y_z)$. Observe that
  the inequality $\sum_{v \in T_z}(1-x_{vz}) \leq K+y_z$ implies that $\sum_{v \in T_z}x_{vz} \geq \sizeof{T_z}-K-y_z$.
  Furthermore, Observation~\ref{obs} together with the inequalities $x_{uv},x_{uz} \leq \alpha$ yields
  \[ 1-x_{vz} \geq 1 - x_{uv} - x_{uz} \geq 1-2\alpha = \beta(2\alpha) \geq \beta x_{vz}. \]
  Thus, we may write
  \begin{align*}
    \sum_{v \in T_z}\wpp_{vz}x_{vz} + \frac{1}{\beta}\sum_{v \in T_z}\wmm_{vz}(1-x_{vz}) &= \sum_{v \in T_z}x_{vz} + \sum_{v \in T_z}(\wpp_{vz} - 1)x_{vz} + \frac{1}{\beta}\sum_{v \in T_z}\wmm_{vz}(1-x_{vz}) \\
    &\geq \sizeof{T} - K - y_v + \sum_{v \in T_z}(\wpp_{vz} + \wmm_{vz} - 1)x_{vz} \\
    &\geq \sizeof{T} - K - y_v,
  \end{align*}
  where the last inequality is based on our assumption that $\wpp_{vz}
  + \wmm_{vz} \geq 1$.  It follows that charging
  $\mu_z[\wpp_{vz}x_{vz} + \frac{1}{\beta}\wmm_{vz}(1-x_{vz})]$ to
  each edge $vz \in E_z$ and charging $\mu_zy_z$ to $z$ itself pays
  for the cluster-penalty of $z$.  \smallskip

  In total, we have paid for all the cluster-costs by making the following charges
  to the edges, where $\muu = \max_{v\neq z \in V}(\mu_v + \mu_z)$: 
  \begin{itemize}
  \item Edges $vz$ within $\{u\} \cup T$ were charged at most
    $\frac{1}{1-2\gamma}\wmm_{vz}(1-x_{vz})$ to pay for their own
    cost. If $\muu > 0$, these edges are also charged at most
    $\muu(\wpp_{vz}x_{vz} + \frac{1}{\beta}\wmm_{vz}(1-x_{vz}))$ to
    pay for the cluster-penalties of their endpoints.  Since
    $\frac{1}{\beta} = \frac{2\alpha}{1-2\alpha}$, their total charge
    is at most $\max\{\muu, \frac{2\alpha\muu}{1-2\alpha} +
    \frac{1}{1-2\gamma}\}$ times their total LP-cost if $\muu > 0$, or
    simply $\frac{1}{1-2\gamma}$ times their LP-cost if $\muu = 0$.
  \item Edges $vz$ for which $v \in \{u\} \cup T$ and $z \in S -
    (\{u\} \cup T)$ were charged at most $\max\{\frac{1}{1-2\gamma},
    \frac{2}{\alpha}\}$ times their LP-cost.
  \item Vertices $z \in \{u\} \cup T$ were charged exactly their LP-cost to
    pay for their cluster-penalty.
  \end{itemize}
  It follows that the clustering $\cee$ has cost at most $c_\alpha\cost(x,y)$,
  where $c_{\alpha}$ is defined as in the statement of the theorem.
\end{proof}
We now determine the optimal value of $\alpha$ in terms of $\muu$ and
$\tau$, under the further hypothesis that $\muu \in (0,4]$. The case
$\muu=0$ will be considered separately. While this hypothesis is not
necessary for the validity of Theorem~\ref{thm:appx}, it simplifies the 
calculation of the approximation ratio. Theorem~\ref{thm:appx} shows
that Algorithm~\ref{alg:round} achieves the approximation ratio
\[ \max\left\{\muu,\ \frac{2\alpha\muu}{1-2\alpha} + \frac{1}{1 - 2\gamma},\ \frac{2}{\alpha}\right\}, \]
where $\gamma = \alpha - \alpha/4\tau$.
Since we have assumed that $\muu \leq 4 \leq 2/\alpha$, the total approximation ratio is minimized when
\[ \frac{2\alpha\muu}{1-2\alpha} +
\frac{1}{1-2\alpha+\frac{\alpha}{2\tau}} = \frac{2}{\alpha\hangs.}\]
It is straightforward to see that for any fixed $\tau \geq 0$ and any
$\muu \in (0,4]$, this cubic equation in $\alpha$ has a unique
solution on $(0, \frac{1}{2})$. Rather than seeking a general
analytical expression for the optimal $\alpha$, which would be too
complicated to be meaningful, we consider particular parameter values
for which there is a simple explicit solution.

When $\muu = 0$, Theorem~\ref{thm:appx} yields the simpler approximation ratio
$\max\{\frac{1}{1-2\gamma}, \frac{2}{\alpha}\}$, so the approximation constant is minimized
when
\[ \frac{1}{1-2\alpha+\frac{\alpha}{2\tau}} = \frac{2}{\alpha\hangs.}\]
Hence, the optimal value of $\alpha$ equals $\frac{2\tau}{5\tau - 1}$, with the 
resulting approximation ratio equal to $5 - (\lfrac{1}{\tau})$. If
additionally $\tau=1$, then we are enforcing a slight generalization of the
probability constraints, and the resulting approximation ratio is $4$,
as in the original paper of Charikar, Guruswami, and
Wirth~\cite{CGW1,CGW2}.  Indeed, in this special case our LP-rounding
algorithm reduces exactly to the aforementioned algorithm.

In the limit $\tau \to \infty$, we are allowing the weights $\wmm_e$
to be arbitrarily large.  In this case, when $\muu > 0$, the total cost
is minimized when $\alpha = \frac{-5 + \sqrt{25 + 16\muu}}{4\muu}$.
When $\muu = 2$, we are effectively enforcing a hard constraint on the
cluster sizes, yielding an approximation ratio roughly equal to
$6.275$. When $\muu = 0$, the total cost is instead minimized when
$\alpha = 2/5$, yielding approximation ratio $5$.

Finally, when $\muu = 2$ and $\tau=1$, we are effectively enforcing
the probability constraints with a hard constraint on the cluster
sizes. In this case, the optimal choice for $\alpha$ is $1/3$, yielding 
approximation ratio $6$. 

Our strategy in this section has been to incorporate the desired size
bounds into both the LP and the rounding procedure. A different
strategy is to solve the LP without any size bound constraints, and
impose the size bounds only in the rounding constraints.  Preliminary
work has shown that constant-factor approximation algorithms can also
be devised along these lines, under the same weight regime, and that
such algorithms can avoid pathological behavior that is sometimes
exhibited by Algorithm~\ref{alg:round}.  Future work will pursue this
line of inquiry further.
 
As a final note, we remark on the complexity of solving the LP and
running Algorithm~\ref{alg:round}.  Interior point LP solvers based on
Karmarkar's method~\cite{karmarkar,interior} require
$O(N^{3.5}\,S^2\,\log S\, \log\log S)$ operations in the worst case,
where $N$ denotes the number of variables, and $S$ denotes the input
size of the problem. For the case of correlation clustering,
$N=O(\sizeof{V(G)}^2)$. The complexity of Algorithm 1 is
$O(\sizeof{V(G)} + \sizeof{E(G)})$, and the higher computational cost
is clearly incurred by the LP solver.

\section{A Random Pivoting Algorithm}\label{sec:pivot}
In this section, we describe a randomized $7$-approximation algorithm
for the unweighted correlation clustering problem with bounded cluster
sizes, based on the \ccpivot{} algorithm of Ailon, Charikar, and
Newman~\cite{ACN1,ACN2}, which is shown in
Algorithm~\ref{alg:ccpivot}. (\ccpivot{} is also known as
\textsc{KwikCluster}.) While this approximation ratio is worse than
the ratio of $6$ achieved in Section~\ref{sec:two-step}, this method
does not require solving a linear program, and has lower asymptotic
complexity than the best known LP solvers. The \ccpivot{} method is
also amenable to parallelization, which is a significant advantage for
large-scale problems encountered in social and biological network
analysis~\cite{parallel}. We discuss parallelization in more detail at
the end of this section.

We now require an \emph{unweighted} instance of the
correlation-clustering problem, and we require hard constraints on the
cluster sizes. Expressed in the language of weights, we require the
following assumptions:
\begin{assumption}\label{asu:ccpivot}
  We assume that $uv \in E(G)$ for every pair of distinct vertices
  $u,v \in V(G)$.  Furthermore, we assume that the edges can be
  partitioned as $E(G) = E^+(G) \cup E^-(G)$, where $\wpp_e = 1$ and
  $\wmm_e = 0$ for all $e \in E^+(G)$, while $\wpp_e = 0$ and $\wmm_e =
  1$ for all $e \in E^-(G)$. Finally, we assume that $\mu_v = 1$ for
  all $v \in V(G)$.
\end{assumption}
For $X \subset E(G)$, we let $d_X(v)$ denote the number of edges of
$X$ incident to $v$, and we let $d^+(v)$ denote $d_{E^+(G)}(v)$. We
also use $N^+(v)$ to denote the set of neighbors of $v \in V(G)$
connected by positively labeled edges.
\begin{algorithm}
\caption{\ccpivot{} algorithm~\cite{ACN1,ACN2}.}
\label{alg:ccpivot}  
\begin{algorithmic}
\STATE{Let $S = V(G)$.}
\WHILE{$S \neq \emptyset$}
\STATE{Pick $v \in S$ uniformly at random.}
\STATE{Let $T = (\{v\} \cup N^+(v)) \cap S$.}
\STATE{Output the cluster $T$.}
\STATE{Let $S = S-T$.}
\ENDWHILE
\end{algorithmic}
\end{algorithm}
\begin{algorithm}
  \caption{Approximation algorithm for correlation clustering with cluster size at most $K+1$.}
  \label{alg:callcc}
  \begin{algorithmic}
    \STATE{Take $X \subset E^+(G)$ to be the smallest set such that $d^+(v) - d_X(v) \leq K$ for all $v$.}
    \STATE{Let $H$ be the labeled graph with $E^+(H) = E^+(G) - X$ and $E^-(H) = E^-(G) \cup X$.}
    \STATE{Run $\ccpivot$ on $H$.}
  \end{algorithmic}
\end{algorithm}

Our algorithm for solving the correlation-clustering problem with
cluster-size constraints, given in Algorithm~\ref{alg:callcc}, is in
some sense the obvious algorithm: we first remove positive edges from
$G$, if necessary, so that every vertex has positive degree at most
$K$, and then run \ccpivot{} on the resulting graph $H$, taking
advantage of the fact that \ccpivot{} always clusters its pivot vertex
$v$ with at most $d^+(v)$ vertices. It is known that the set $X$
required in the first step can be found in polynomial
time~\cite{Gabow,HellKirkpatrick,LovaszPlummer} -- more precisely,
that it can be found in time
$O(\sqrt{K\sizeof{V(G)}}\sizeof{V(G)}^2)$.  Since \ccpivot{} itself
runs in time $O(V(G) + E(G))$, the running time of
Algorithm~\ref{alg:callcc} is dominated by finding the set $X$.
\begin{theorem}\label{thm:7appx}
  Algorithm~\ref{alg:callcc} is an expected $7$-approximation algorithm.
\end{theorem}
\begin{proof}
  Let $\opt_G$ and $\opt_H$ denote the optimal size-bounded clutering
  costs in $G$ and $H$, respectively. Observe that
  $\opt_G \geq \opt_H - \sizeof{X}$, since for any clustering $\cee$,
  we have $\cost_G(\cee) \geq \cost_H(\cee) - \sizeof{X}$.
  Furthermore, $\opt_G \geq \sizeof{X}$, since the positive edges
  contained within clusters constitute a subgraph of maximum degree
  at most $K$. Taking a convex combination of these lower bounds
  yields the following lower bound on $\opt_G$:
  \begin{equation}
    \label{eq:seven}
    \opt_G \geq \frac{3}{7}(\opt_H - \sizeof{X}) + \frac{4}{7}\sizeof{X} = \frac{1}{7}(3\opt_H + \sizeof{X}).
  \end{equation}
  On the other hand, for any clustering $\cee$, we have $\cost_G(\cee) \leq \cost_H(\cee) + \sizeof{X}$, and we know
  (by Ailon--Charikar--Newman) that $\ex[\cost_H(\cee)] \leq 3\opt_H$. Hence,
  \[ \ex[\cost_G(\cee)] \leq \ex[\cost_H(\cee)] + \sizeof{X} \leq 3\opt_H + \sizeof{X} \leq 7\opt_G, \]
  where the last inequality follows from Inequality~\ref{eq:seven}.
\end{proof}

We now turn our attention to the problem of parallelizing
Algorithm~\ref{alg:callcc}. There are two steps that must be
parallelized: the step of finding the set $X$ of positive edges to
ignore, and the \ccpivot{} step. To carry out the \ccpivot{} step in
parallel, we can use the MapReduce-based algorithm of Chierchetti,
Dalvi, and Kumar~\cite{mapreduce}.  Like \ccpivot{}, this algorithm is
randomized, and it provides \emph{expected} approximation guarantees,
rather than deterministic ones.

For any $\epsilon \in (0, 1/7)$, the MapReduce-based algorithm
achieves an expected approximation ratio of $3 +
\frac{14\epsilon}{1-7\epsilon}$, terminating with high probability
after 
$$O\left(\frac{1}{\epsilon}\log(\sizeof{V(G)})\log(\Delta^+(G))\right)$$ 
rounds of its main loop, where $\Delta^+(G)$ is the maximum positive degree
of $G$. Thus, in the context of Algorithm~\ref{alg:callcc}, the parallel \ccpivot{} step terminates with high
probability after $O(\frac{1}{\epsilon}\log\sizeof{V(G)}\log K)$
rounds.

Since the runtime of Algorithm~\ref{alg:callcc} is dominated by
finding the set $X$, there is little benefit to parallelizing the call
to \ccpivot{} without also parallelizing this step. We are not aware
of any existing parallel algorithm for solving this problem exactly,
but it is easy to obtain a parallelized $2$-approximation to the
optimal solution: allow each vertex $v$ to arbitrarily and
independently choose $K$ incident positive edges, and set all its
other incident edges to negative edges. The total number of positive
edges switched this way is at most $\sum_{v \in V(G)}(d(v) - K)$;
possibly fewer edges are switched, since we ``save'' an edge whenever
two adjacent vertices both choose to delete the edge joining
them. Since any set $X$ of the desired form must delete at least
$\frac{1}{2}\sum_{v \in V(G)}(d(v)-K)$ edges, this is a
$2$-approximation to the desired set.

Since we are only finding a $2$-approximation to the optimal set $X$
instead of the optimal set, the approximation ratio of the overall
algorithm increases. Assuming that the \ccpivot{} step can be carried out with
approximation ratio $r$, we obtain an approximation ratio of $3r+2$ by modifying
the proof of Theorem~\ref{thm:7appx}. Instead of the bound $\opt_G \geq \sizeof{X}$,
we now have $\opt_G \geq \frac{1}{2}\sizeof{X}$ . This yields the bound
\begin{align*}
  \opt_G &\geq \frac{r}{3r+2}(\opt_H - \sizeof{X}) + \frac{2r+2}{3r+2}\left(\frac{1}{2}\sizeof{X}\right) \\
  &= \frac{1}{3r+2}(r\opt_H + \sizeof{X}) \\
  &\geq \frac{1}{3r+2}\ex[\cost_G(\cee)].  
\end{align*}
As a result, the approximation ratio of $3 + O(\epsilon)$ in the parallel \ccpivot{} algorithm
yields an approximation ratio of $11 + O(\epsilon)$ in the parallel version of Algorithm~\ref{alg:callcc}.
\section{Acknowledgments}
The authors thank the anonymous referees for their patient, careful reading and for their helpful comments, and the editor for his timely handling of the manuscript.
\bibliographystyle{amsplain}
\bibliography{cluster,cluster1}
\end{document}